%% file: main.tex
\documentclass[journal,twoside,web]{ieeecolor}
\usepackage{generic}
\usepackage{lipsum}
\usepackage{cite}
\usepackage{amsmath,amssymb,amsfonts}
\usepackage{algorithmic}
\usepackage{graphicx}
\usepackage{textcomp}
\usepackage{orcidlink}
\usepackage{tikz}
\usetikzlibrary{fit}
\def\BibTeX{{\rm B\kern-.05em{\sc i\kern-.025em b}\kern-.08em
    T\kern-.1667em\lower.7ex\hbox{E}\kern-.125emX}}
\DeclareMathOperator{\sgn}{sgn}

\DeclareMathOperator{\subjto}{subject\,to}
\DeclareMathOperator{\minimize}{minimize}

\DeclareMathOperator{\trace}{trace}
\DeclareMathOperator{\diag}{diag}

\newcommand{\set}[1]{#1}
\newcommand{\est}[1]{\hat{#1}}
\newcommand{\mat}[1]{\mathbf{#1}}
\renewcommand{\vec}[1]{\underline{#1}} 

\newtheorem{definition}{Definition}
\newtheorem{theorem}{Theorem}
\newtheorem{lemma}{Lemma}

\begin{document}
\title{An Event-Based Approach for the Conservative Compression of Covariance Matrices}
\author{\orcidlink{0000-0001-6268-678X}\,Christopher Funk${}^*$, \orcidlink{0000-0001-8996-5738}\,Benjamin Noack${}^\dagger$
\thanks{Submitted for review on December 8, 2023. © 2024 IEEE. Personal use of this material is permitted. Permission from IEEE must be obtained for all other uses, in any current or future media, including reprinting/republishing this material for advertising or promotional purposes, creating new collective works, for resale or redistribution to servers or lists, or reuse of any copyrighted component of this work in other works.
}
\thanks{${}^*$Christopher Funk is with the Autonomous Multisensor Systems Group, Institute for Intelligent Cooperating Systems, Otto von Guericke University Magdeburg, Germany (e-mail: christopher.funk@ovgu.de). Corresponding author.}
\thanks{${}^\dagger$Benjamin Noack heads the Autonomous Multisensor Systems Group, Institute for Intelligent Cooperating Systems, Otto von Guericke University Magdeburg, Germany (e-mail: benajmin.noack@ieee.org).}}

\maketitle

\begin{abstract}
    This work introduces a flexible and versatile method for the data-efficient yet conservative transmission of covariance matrices, where a matrix element is only transmitted if a so-called triggering condition is satisfied for the element.
    Here, triggering conditions can be parametrized on a per-element basis, applied simultaneously to yield combined triggering conditions or applied only to certain subsets of elements.
    This allows, e.g., to specify transmission accuracies for individual elements or to constrain the bandwidth available for the transmission of subsets of elements.
    Additionally, a methodology for learning triggering condition parameters from an application-specific dataset is presented. 
    The performance of the proposed approach is quantitatively assessed in terms of data reduction and conservativeness using estimate data derived from real-world vehicle trajectories from the InD-dataset, demonstrating substantial data reduction ratios with minimal over-conservativeness. 
    The feasibility of learning triggering condition parameters is demonstrated.
\end{abstract}


\section{Introduction}
Estimate fusion is a fundamental operation in sensor fusion systems that involves the combination of multiple estimates, each accompanied by its corresponding covariance matrix, to yield an enhanced single estimate. This operation can be conducted either in a centralized manner, where a fusion center combines individual estimates, or in a decentralized fashion without a designated central entity. The significant bandwidth required for transmitting both the estimates and covariance matrices that encode estimate uncertainty to the fusion center or other entities, motivates the use of (lossy) compression techniques, in particular for covariance matrices, as they tend to dominate the amount of transmitted data. 
In addition, in safety-critical applications such as automated driving, the covariance matrices associated with both the original estimates and the fused estimate must not underestimate the true uncertainty, i.e., they must be conservative.
Underestimating the uncertainty, e.g., of a position estimate, could potentially lead to safety hazards and accidents. As such, any lossy compression methods employed must guarantee that the compressed covariance matrices uphold the integrity of the uncertainty encoded in the original covariance matrices.

\begin{figure}
    \centering
    \includegraphics{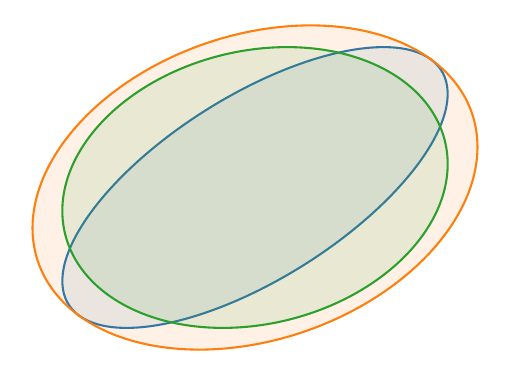}
    \caption{Confidence / credible regions associated with covariance matrix $\mat P$ (blue), covariance matrix $\hat{\mat P}$ conservative w.r.t. $\mat P$ (orange), and covariance matrix $\tilde{\mat P}$ not conservative w.r.t. $\mat P$ (green). As can be seen the region of $\hat{\mat P}$ includes that of $\mat P$ while that of $\tilde{\mat P}$ does not. The latter corresponds to an underestimation of the uncertainty encoded by $\mat P$. \label{fig:cons_condition}}
\end{figure}

To address the issue of compressed transmission of covariance matrices, various techniques have been proposed. Some works have investigated diagonal approximations \cite{forsling_consistent_2019} or elementwise quantization of covariance matrices using diagonal dominance \cite{funk_conservative_2020-1} and modified Cholesky decompositions \cite{funk_conservative_2021-1} to ensure data reduction while preserving conservativeness.
Linear projections combined with the covariance intersection algorithm were used in \cite{forsling_communication_2020} to achieve conservative covariance matrix compression.
In addition, methods to compute the optimal linear projections in terms of the mean square error have been proposed \cite{forsling_optimal_2022}.
Issues regarding the practicality of the aformentioned approaches in decentralized fusion scenarios were remedied in \cite{forsling_decentralized_2023}.
Unrelated to the above, and most relevant for this work, an event-triggered approach \cite{funk_conservative_2023} has emerged as a promising solution, allowing for the selective transmission of covariance information based on its novelty, thereby reducing the transmitted data volume. 
This approach utilizes the implicit information inherent in the decision not to transmit data to ensure effective data compression while maintaining conservativeness.
Works prior to \cite{funk_conservative_2023} have considered the event-triggered transmission of estimates \cite{battistelli_distributed_2018}, measurements \cite{sijs_event_2012,sijs_event-based_2013,han_stochastic_2015}, and system inputs \cite{noack_state_2020} while also leveraging implicit information to improve accuracy or provide error bounds.

This work is an extension of \cite{funk_conservative_2023}, where an event-triggered covariance matrix transmission approach was introduced, which sends individual covariance matrix elements based on a triggering condition being satisfied for the respective element.
Here, the approach is rigorously rederived, its theory generalized and extended, and its experimental evaluation enhanced.
The result is a more flexible and versatile approach for the conservative data-efficient transmission of covariance matrices for sensor fusion.
The contributions of this work are:
\begin{enumerate}
    \item The rigorous rederivation of the approach from \cite{funk_conservative_2023}, thereby generalizing it to allow for different triggering condition parameters for different elements, for the simultaneous use of multiple triggering rules for a single element, and for the application of triggering conditions only to certain subsets of covariance matrix elements.
    \item A new scale-invariant triggering condition, the 'relative-change' trigger, which simplifies the selection of adequate triggering condition parameters in scenarios where covariance matrix sequences of varying scales can occur.
    \item An approach to learning optimal triggering condition parameters from an application-specific dataset of covariance matrix sequences and
    \item a quantitative performance evaluation on a dataset derived by performing an estimation task on real-world data that replaces the qualitative evaluation performed in \cite{funk_conservative_2023} on a rather limited set of synthetic data.
\end{enumerate}

\section{Notation}
In the following, scalars are denoted by standard lowercase letters (e.g., $a$), vectors by lowercase underlined letters (e.g., $\vec a$), and matrices by uppercase bold letters (e.g., $\mat A$).
Elements of a vector $\vec a$ or a matrix $\mat A$ are indicated by $[\vec a]_i$ or $[\mat A]_{ij}$, where $i$ and $j$ are the element's row- and column index, respectively.
The identity matrix is denoted by $\mat I$ and the all one vector and matrix by $\vec 1$ and $\mat 1$, respectively.
Elementwise multiplication (the Hadamard product) of two matrices $\mat A$ and $\mat B$ is indicated by $\mat A\odot\mat B$.
Furthermore, $|\mat A|$ denotes taking the elementwise absolute value of the matrix $\mat A$.
For symmetric matrices $\mat A$ and $\mat B$ the notation $\mat A\succeq\mat B$ is synonymous with $\mat A-\mat B$ being positive semidefinite.
The notation $\mat A\geq\mat B$ expresses the elementwise (standard) inequality of the matrices $\mat A$ and $\mat B$.
Sets are denoted by standard uppercase letters, e.g., $\set A$, and the set cardinality by $|\set A|$.
The set of real symmetric $n\times n$ matrices is denoted as $\set S_n$, the set of real positive semidefinite $n\times n$ matrices as $\set S_n^+$.
The set of real symmetric diagonally dominant $n\times n$ matrices is denoted as $\set{DD}_n$.

\section{Conservative Elementwise Event-Triggered Covariance Matrix Compression}\label{sec:event_triggered_approach}

This section is dedicated to the introduction of the approach first proposed in \cite{funk_conservative_2023}, to its rigorous derivation, generalization, and extension.
The issue of achieving data reduction while retaining conservativeness is adressed using robust optimization and the notion of diagonal dominance.
A new general condition that event-triggering mechanisms used by the generalized approach must satisfy is presented, and several concrete such mechanisms, among them extensions of those in \cite{funk_conservative_2023} and new ones, are introduced. 
Finally, the novel idea of per-element application and combination of the different triggering mechanisms is introduced and discussed.
The proposed extensions and additions allow, e.g., to control the accuracy of transmission for each covariance matrix element individually or to transmit certain subsets of elements only within an allotted amount of bandwidth.
Note that all proofs for theorems, lemmas, etc., are located in the appendix.

\subsection{Overview of Structure and Function}\label{sec:overview_trigger_bounder}

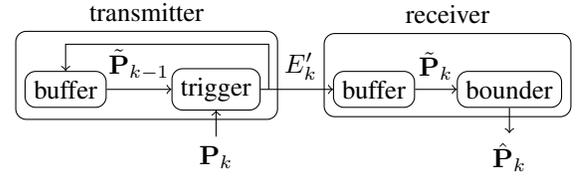
\begin{figure}
    \centering
    \input{approach_structure.tex}
    \caption{Overview of event-triggered covariance compression approach. Covariance matrix to be transmitted $\mat P_k$, transmitter buffer matrix $\tilde{\mat P}_{k-1}$, reduced event-set $\set E_k'$, receiver buffer matrix $\tilde{\mat P}_k$ and upper bound $\hat{\mat P}_k$ on original covariance matrix $\mat P_k$.\label{fig:event_approach}}
\end{figure}

The proposed approach investigates the scenario in which a transmitter aims to send a sequence of covariance matrices $\mat P_k\in\set S_n^+$, $k\in\mathbb N$, to a receiver in an elementwise event-triggered manner, which will subsequently be elucidated. 
Note that in the following, it is assumed that no packet loss occurs, i.e., data that is transmitted is received completely and without errors.
In practice this may be achieved through protocols such as TCP, which guarantee data delivery, and error correcting codes. 
For a structural overview of the approach see Fig.~\ref{fig:event_approach}, which illustrates the components -- trigger, buffers, and bounder, distributed across transmitter and receiver -- and the data flow between them. 
Fig.~\ref{fig:example_cov_transmission} shows an example of an event-triggered covariance matrix sequence transmission that uses the proposed approach.

The transmission proceeds as follows: The trigger located in the transmitter determines in each timestep $k$ the index-value pairs $((i,j),[\mat P_k]_{ij})$ of the current covariance matrix $\mat P_k$ to be transmitted, based only on $\mat P_k$ and the matrix $\tilde{\mat P}_{k-1}\in\set S_n$ stored in the buffer located in the transmitter. 
This buffer matrix is an approximation of the actual covariance matrix at the last timestep.
The index-value pairs to be transmitted in the timestep are collected in the event-set $\set E_k$.
It is assumed that the trigger selects the index-value pairs symmetrically, i.e., such that $((i,j),[\mat P_k]_{ij})\in\set E_k$ if and only if $((j,i),[\mat P_k]_{ji})\in\set E_k$.
The so-called reduced event-set is denoted $\set E_k'$ and contains only those index-value pairs from $\set E_k$ that correspond to the upper triangle of $\mat P_k$.
Due to symmetry, $\set E_k$ can be reconstructed from $\set E_k'$.
The reduced event-set $E_k'$ is transmitted and the transmitter buffer is updated using the index-value pairs in the reconstructed event-set $\set E_k$, yielding the new buffer matrix $\tilde{\mat P}_k\in S_n$.

Simultaneously, the receiver operates its own buffer, which retains the most recent covariance matrix elements until updated values are received in the form of $\set E_k'$ and incorporated into the buffer. 
Due to the assumption of no packet loss and error-free transmission, the values of transmitter and receiver buffer are identical at each timestep, as long as they are initialized identically.
Notably, the buffers may not represent a valid covariance matrix, or a conservative one, i.e., one satisfying $\tilde{\mat P}_k\succeq\mat P_k$. 
See Fig.~\ref{fig:cons_condition} for a visual explanation of this conservativeness condition in terms of confidence / credible regions.
To address the aforementioned issue, the approach incorporates the bounder, which is located in the receiver, and is specifically designed to generate an upper bound $\hat{\mat P}_k\succeq\mat P_k$, $\hat{\mat P}_k\in\set S_n^+$, on the original covariance matrix $\mat P_k$. 
The bounding mechanism relies solely on the present (and possibly past) information $\tilde{\mat P}_k,\tilde{\mat P}_{k-1}$ within the buffer and the event-set $\set E_k$, which are available at the receiver.

\begin{figure}
    \centering
    \input{example_transmission.tex}
    \caption{Exemplary event-triggered covariance transmission. Sent covariance matrix elements ($[\mat P_k]_{ij}$), changed buffer matrix elements ($[\tilde{\mat P}_k]_{ij}$), and bound elements ($[\est{\mat P}_k]_{ij}$) modified by the bounder compared to the buffer value are shown in red.\label{fig:example_cov_transmission}}
\end{figure}

\subsection{Bounder Using Diagonal Dominance}\label{sec:bounder_cov}

The goal of the bounder is to find the tightest $\hat{\mat P}_k$, in the sense that $\trace(\hat{\mat P}_k)$ is minimal\footnote{Note that $\trace(\cdot)$ is matrix monotone, i.e., $\mat A\preceq\mat B\implies\trace(\mat A)\leq\trace(\mat B)$. This  justifies minimizing the trace to find a minimal upper bound.}, that is still conservative, i.e., satisfies $\hat{\mat P}_k\succeq\mat P_k$, given only the present and past buffer matrices $\tilde{\mat P}_k,\tilde{\mat P}_{k-1}$ and the event-set $\set E_k$.
In the following, it is assumed, without loss of generality, that the bound is of the form $\hat{\mat P}_k=\tilde{\mat P}_k+\mat S_k$, where $\mat S_k\in\set S_n$ is to be determined.
With this, the conservativeness requirement is expressed as $\hat{\mat P}_k=\tilde{\mat P}_k+\mat S_k\succeq\mat P_k$ or equivalently as $\mat\Delta_k+\mat S_k\succeq\mat 0$, where $\mat\Delta_k=\tilde{\mat P}_k-\mat P_k$ is the buffer error.
Since the actual buffer error $\mat\Delta_k$ is unknown to the receiver, the above conservativeness requirement is only guaranteed to be satisfied if $\mat\Delta+\mat S_k\succeq\mat 0$ holds for all possible buffer errors $\mat\Delta$.
The set of possible buffer errors $\set F_k$ is determined by the employed triggering rule, the buffer matrices $\tilde{\mat P}_k,\tilde{\mat P}_{k-1}$, and the event-set $\set E_k$. 
In the following, $\set F_k$ is assumed to be of the elementwise form 
\begin{align}
    &\set F_k =\left\{\mat\Delta\in\set S_n\middle||\mat\Delta|\leq\hat{\mat\Delta}_k\right\}\label{eq:possible_errors}
\end{align}
where $\hat{\mat\Delta}_k\in\set S_n$ are non-negative elementwise bounds that are derived from the triggering rule and depend only on $\tilde{\mat P}_k,\tilde{\mat P}_{k-1}$ and $\set E_k$.
Any triggering rule that results in $\set F_k$ having the above form can be used with the proposed approach.
Note that if an index-value pair $((i,j),[\mat P_k]_{ij})$ is in the event-set $\set E_k$, the buffer error $[\mat\Delta_k]_{ij}$ will be zero so that the bound $[\hat{\mat \Delta}_k]_{ij}$ can be set to zero as well.
As will be shown in this section, the particular form of $\set F_k$ enables the bounder to be efficient while still allowing for a variety of trigger rules, c.f. Sec.~\ref{sec:triggers}.

Having introduced $\set F_k$, the overall task of minimizing the trace of the bound, subject to the conservativeness constraint, is succinctly formulated as the robust optimization problem
\begin{alignat}{2}
    \underset{\mat S_k\in\set S_n}{\minimize} & \quad \trace(\tilde{\mat P}_k+\mat S_k)\label{eq:bound_sdp}                      \\
    \subjto                                   & \quad\mat\Delta+\mat S_k\succeq\mat 0\quad \forall\mat\Delta\in\set F_k\nonumber
\end{alignat}
where the optimal bound is obtained by setting $\hat{\mat P}_k=\tilde{\mat P}_k+\mat S_k^*$, where $\mat S_k^*$ is a globally optimal solution of \eqref{eq:bound_sdp}. 
Unfortunately, solving \eqref{eq:bound_sdp} is hard as it is a semi-infinite semidefinite program \cite[Sec.\,4]{ben-tal_robust_2002}. 
In the subsequent part of this section, the complexity of the problem is reduced through approximation to ensure a manageable solution. 
To achieve this, the method proposed in \cite{ahmadi_dsos_2019} is adopted, wherein the positive semidefinite (PSD) constraints of an optimization problem are replaced by so-called diagonal dominance (DD) constraints.
\begin{definition}
    \label{def:diagonal_dominance}
    $\mat A\in\mathbb R^{n\times n}$ is diagonally dominant (DD) if $[\mat A]_{ii}\geq\sum_{\substack{j=1\\j\neq i}}^n|[\mat A]_{ij}|$ for all $i=1,\ldots,n$.\footnote{Diagonal dominance is often defined slightly differently by requiring $|[\mat A]_{ii}|\geq\sum_{j\neq i}|[\mat A]_{ij}|$ for all $i=1,\ldots,n$.
    In conjunction with the condition $[\mat A]_{ii}\geq 0$ for all $i=1,\ldots,n$ this is equivalent to our definition.}
\end{definition}
The significance of diagonal dominance in approximating \eqref{eq:bound_sdp} is due to its relation to positive semidefiniteness, as substantiated by the following lemma.
\begin{lemma}
    \label{thm:dd_implies_psd}
    If $\mat A\in\set{DD}_n$ then it is PSD.
\end{lemma}
\begin{proof}
    See Appendix\,\ref{sec:dd_implies_psd}.
\end{proof}
In the context of approximating \eqref{eq:bound_sdp}, the preceding theorem demonstrates that $\mat\Delta+\mat S_k\in\set DD_n$ implies $\mat\Delta+\mat S_k\succeq\mat 0$ and thus replacing the PSD constraints with DD constraints preserves the conservativeness of the solutions of \eqref{eq:bound_sdp}.
The robust optimization problem with DD constraints substituted for its PSD constraints, which is considered from now on, is
\begin{alignat}{2}
    \underset{\mat S_k\in\set S_n}{\minimize} & \quad \trace(\tilde{\mat P}_k+\mat S_k)\label{eq:bound_dd}                 \\
    \subjto                                   & \quad 
    \begin{matrix}
    [\mat\Delta]_{ii}+[\mat S_k]_{ii}\geq\sum_{\substack{j=1\\j\neq i}}^n|[\mat\Delta]_{ij}+[\mat S_k]_{ij}|\\
    \forall\mat\Delta\in\set F_k\nonumber,i=1,\ldots,n
    \end{matrix}
\end{alignat}
where the DD constraints are shown explicitly.
It is a semi-infinite linear program (LP), and can be solved using existing methods \cite[Sec.\,2.2]{bertsimas_theory_2010}.
However, \eqref{eq:bound_dd} can be simplified further: The feasible set can be represented using a finite instead of an infinite number of linear inequality constraints.
\begin{lemma}\label{thm:feasible_sets_identical}
    Let $\set F_k$ be of the form \eqref{eq:possible_errors}, then the sets $$\set A=\left\{\mat S_k\in\set S_n\middle|\begin{matrix}[\mat\Delta]_{ii}+[\mat S_k]_{ii}\geq\sum_{\substack{j=1\\j\neq i}}^n|[\mat\Delta]_{ij}+[\mat S_k]_{ij}|,\\\forall\mat\Delta\in\set F_k,i=1,\ldots,n\end{matrix}\right\}$$ and $$\set B=\left\{\mat S_k\in\set S_n\middle|\begin{matrix}[\mat S_k]_{ii}\geq \sum_{j=1}^n[\hat{\mat\Delta}_k]_{ij}+\sum_{\substack{j=1\\j\neq i}}^n|[\mat S_k]_{ij}|,\\i=1,\ldots,n\end{matrix}\right\}$$ are identical.
\end{lemma}
\begin{proof}
    See Appendix \ref{sec:feasible_sets_identical}.
\end{proof}
This result shows that the semi-infinite LP \eqref{eq:bound_dd}, due to the special form of $\set F_k$, is actually equivalent to the standard LP
\begin{alignat}{2}
    \underset{\mat S_k\in\set S_n}{\minimize} & \quad \trace(\tilde{\mat P}_k+\mat S_k)\label{eq:bound_lp}                 \\
    \subjto                                   & \quad 
    \begin{matrix}
    [\mat S_k]_{ii}\geq\sum_{j=1}^n[\hat{\mat\Delta}_k]_{ij}+\sum_{\substack{j=1\\j\neq i}}^n|[\mat S_k]_{ij}|,\\
    i=1,\ldots,n
    \end{matrix}\nonumber
\end{alignat}
which has a closed form solution, as the next theorem shows.
\begin{theorem}\label{thm:bound_lp_closed_form_solution}
    A global minimizer of \eqref{eq:bound_lp} is given by
    \begin{equation*}
        [\mat S_k^*]_{ij} =
        \begin{cases}
            \sum_{j=1}^n[\hat{\mat\Delta}_k]_{ij} & ,\,i=j              \\
            0                           & ,\,\text{otherwise}
        \end{cases},
    \end{equation*}
    where $i,j=1,\ldots,n$. More succinctly, $\mat S_k^*=\diag(\hat{\mat\Delta}_k\vec 1)$.
\end{theorem}
\begin{proof}
    See Appendix\,\ref{sec:bound_lp_closed_form_solution}.
\end{proof}
The bound $\hat{\mat P}_k=\tilde{\mat P}_k+\mat S_k^*$ derived from the above minimizer is extremely simple to compute and implement and thus clearly preferable to solving the semi-infinite LP \eqref{eq:bound_dd} using existing iterative numerical methods.
In addition, one can derive a general upper bound on the error $\|\hat{\mat P}_k-\mat P_k\|_F$ in terms of $\hat{\mat\Delta}_k$, where $\|\cdot\|_F$ is the Frobenius norm.
\begin{theorem}\label{thm:bound_error_bound}
    Let $\hat{\mat P}_k=\tilde{\mat P}_k+\mat S_k^*$, where $\mat S_k^*$ is an optimal solution of \eqref{eq:bound_lp}. Then the bound error is upper bounded as
    \begin{equation*}
        \|\mat W\odot(\hat{\mat P}_k-\mat P_k)\|_F\leq\|\mat W\odot(\diag(\hat{\mat\Delta}_k\vec 1)+\hat{\mat\Delta}_k)\|_F,
    \end{equation*}
    where $\mat W\in\{0,1\}^{n\times n}$ allows selecting submatrices.
\end{theorem}
\begin{proof}
    See Appendix\,\ref{sec:bound_error_bound}.
\end{proof}
The use of the Frobenius norm allows to derive error bounds for individual elements as a special case.
Note that the above error bound depends on $\hat{\mat\Delta}_k$ and thus may change over time.
Because $\hat{\mat\Delta}_k$ only depends on $\tilde{\mat P}_k,\tilde{\mat P}_{k-1}$ and $E_k$, which are both known at the receiver, the error bound can be computed there and at the transmitter without any issues.

\subsection{Basic Triggers}\label{sec:triggers}

\begin{figure}
    \centering
    \begin{tikzpicture}
        \node (abch) at (0,0) {$\begin{bmatrix} \textcolor{red}{1.0} & \textcolor{red}{0.5} & \textcolor{red}{0.7}\\ \textcolor{red}{0.5} & \textcolor{red}{0.9} & 0.3\\ \textcolor{red}{0.7} & 0.3 & \textcolor{red}{1.1} \end{bmatrix}$};
        \node[below of=abch,align=center] {absolute-change\\ $\mat T=0.5\cdot\mat 1$};
        \node (moch) at (3,0) {$\begin{bmatrix} \textcolor{red}{1.0} & 0.5 & 0.7\\ 0.5 & \textcolor{red}{0.9} & 0.3\\ 0.7 & 0.3 & \textcolor{red}{1.1} \end{bmatrix}$};
        \node[below of=moch,align=center] {relative-change\\ $\mat T=0.75\cdot\mat 1$};
        \node (hybr) at (6,0) {$\begin{bmatrix} \textcolor{red}{1.0} & 0.5 & \textcolor{red}{0.7}\\ 0.5 & \textcolor{red}{0.9} & 0.3\\ \textcolor{red}{0.7} & 0.3 & \textcolor{red}{1.1} \end{bmatrix}$};
        \node[below of=hybr,align=center] {$N$-most-changed\\ $N=4$};
    \end{tikzpicture}
    \caption{Matrix $|\mat P_k-\tilde{\mat P}_{k-1}|$ ($|\mat P_k-\tilde{\mat P}_{k-1}|/|\tilde{\mat P}_{k-1}|$ for relative-change trigger) based on which the trigger decision is made and the elements in $\set E_k$ (red) for different triggers. Here, the $N$-most-changed trigger uses the absolute buffer deviation. Symmetric elements need not be transmitted. Hence, although $|E_k|=5$ for the $N$-most-changed trigger, only $N=4$ elements are sent. \label{fig:different_triggers_illustration}}
\end{figure}

In the previous section, the practical and computational feasibility of the bounder was achieved by constraining the potential buffer errors to the specific form given in \eqref{eq:possible_errors}.
There, $\hat{\mat\Delta}_k\in\set S_n$ are non-negative elementwise bounds on the possible buffer error that rely solely on the receiver's buffer matrices $\tilde{\mat P}_k,\tilde{\mat P}_{k-1}$ and the received event-set $\set E_k$.
In this section, concrete triggering mechanisms are introduced that precisely adhere to the aforementioned format.

\subsubsection*{Absolute-Change Trigger}
The first and simplest considered mechanism is the 'absolute-change trigger'. 
Under this mechanism an index-value tuple $((i,j),[\mat P_k]_{ij})$ is in the transmitted event-set $\set E_k$ if and only if $|[\mat P_k-\tilde{\mat P}_{k-1}]_{ij}|>[\mat T]_{ij}$ is satisfied. 
Here, $\mat T\in\set S_n$ is a predetermined matrix comprising non-negative thresholds.
The absolute-change trigger transmits an index-value pair if the value sufficiently deviates, in absolute value, from the value stored in the transmitter buffer.
\begin{theorem}[Absolute-Change Bounds]\label{thm:absolute_change_bounds}
    Using the absolute-change trigger, bounds on possible buffer errors are given by
    \begin{align*}
        [\hat{\mat\Delta}_k]_{ij} = 
        \begin{cases}
            [\mat T]_{ij} & ,\,((i,j),\cdot)\notin\set E_k\\
            0 & ,\,\text{otherwise}
        \end{cases},
    \end{align*}
    where $\mat T\in\set S_n$ are the non-negative trigger thresholds. 
\end{theorem}
\begin{proof}
    See Appendix\,\ref{sec:absolute_change_bounds}.
\end{proof}
In contrast to \cite{funk_conservative_2023}, the absolute-change trigger as introduced here allows for individual thresholds on a per-element basis.
The error $\|\mat W\odot(\hat{\mat P}_k-\mat P_k)\|_F$ that can occur is seen from Theorems~\ref{thm:bound_error_bound} and \ref{thm:absolute_change_bounds} to be upper bounded by ${\|\mat W\odot(\diag(\mat T\vec 1)+\mat T)\|_F}$ in the worst case, i.e., when $\set E_k$ is empty.
Here, $\mat W\in\{0,1\}^{n\times n}$ can be used to select arbitrary submatrices of $\hat{\mat P}_k-\mat P_k$, including individual elements.
Note that $|\set E_k'|$ and thus the number of elements that need to be transmitted is bounded only by the number of elements in the upper triangle of $\mat P_k$.

\subsubsection*{Relative-Change Trigger}
The absolute-change trigger is not invariant with respect to the absolute scale of the covariance matrices intended for transmission. 
That is, the same set of thresholds results in different amounts of data being transmitted even for covariance matrix sequences that are identical in all but scale.
This makes it challenging to select appropriate thresholds that work well for the variety of covariance matrix sequences that can occur even within the same application.
The 'relative-change trigger' addresses this concern by incorporating an index-value pair $((i,j),[\mat P_k]_{ij})$ into $\set E_k$ when the condition $|[\mat P_k-\tilde{\mat P}_{k-1}]_{ij}|>[\mat T]_{ij}|[\tilde{\mat P}_{k-1}]_{ij}|$ is met. 
Here, the trigger operates based on the relative deviations of the covariance matrix elements from the transmitter buffer elements exceeding predefined non-negative thresholds $\mat T\in\set S_n$.
As the trigger focuses on relative changes, the scaling of the covariance matrix sequence does not affect the frequency of triggered events.
\begin{theorem}[Relative-Change Bounds]\label{thm:relative_change_bounds}
    Using the relative-change trigger, bounds on possible buffer errors are given by
    \begin{align*}
        [\hat{\mat\Delta}_k]_{ij} = 
        \begin{cases}
            [\mat T]_{ij}|[\tilde{\mat P}_k]_{ij}| & ,\,((i,j),\cdot)\notin\set E_k\\
            0 & ,\,\text{otherwise}
        \end{cases},
    \end{align*}
    where $\mat T\in\set S_n$ are the non-negative trigger thresholds. 
\end{theorem}
\begin{proof}
    See Appendix\,\ref{sec:relative_change_bounds}.
\end{proof}
Given the dependence of the elementwise bounds $\hat{\mat\Delta}_k$ of the relative-change trigger on the buffer matrix, it is not possible to establish a time-invariant error bound as for the absolute-change trigger.
Instead the error bound depends on the value of the buffer: 
\begin{align*}
    \|\mat W\odot(\hat{\mat P}_k-\mat P_k)\|_F &\leq\\
    \|\mat W&\odot(\diag((\mat T\odot|\tilde{\mat P}_k|)\vec 1)+\mat T\odot|\tilde{\mat P}_k|)\|_F.
\end{align*}
Here, as before $\mat W\in\{0,1\}^{n\times n}$ can be used to select arbitrary submatrices of $\hat{\mat P}_k-\mat P_k$ to compute an upper bound on.
Similarly to the absolute-change trigger, no bound but the trivial one can be given on the number of transmitted elements.

\subsubsection*{$N$-Most-Changed Trigger}
There are situations where adherence to a specific data transmission budget is necessary, or where the transmission of solely the most significant elements is of interest. In such instances, the '$N$-most-changed trigger' proves beneficial. It involves the inclusion of the index-value pairs $((i,j),[\mat P_k]_{ij})$ in $\set E_k$ when the absolute deviation $|[\mat P_k-\tilde{\mat P}_{k-1}]_{ij}|$ (or relative deviation $|[\mat P_k-\tilde{\mat P}_{k-1}]_{ij}|/|[\tilde{\mat P}_{k-1}]_{ij}|$) from the transmitter buffer value ranks among the $N$ largest deviations in the upper triangle of $\mat P_k$.
\begin{theorem}[$N$-Most-Changed Bounds]\label{thm:most_changed_bounds}
    Using the $N$-most-changed trigger, bounds on possible buffer errors are given by
    \begin{align*}
        [\hat{\mat\Delta}_k]_{ij} = 
        \begin{cases}
            \hat\Delta_k & ,\,((i,j),\cdot)\notin\set E_k\\
            0 & ,\,\text{otherwise}
        \end{cases},
    \end{align*}
    where $\hat\Delta_k=\min_{((l,m),\cdot)\in\set E_k} |[\tilde{\mat P}_k]_{lm}-[\tilde{\mat P}_{k-1}]_{lm}|$ if the absolute deviation is used as basis of the trigger decision and 
    \begin{align*}
        \hat\Delta_k=|[\tilde{\mat P}_k]_{ij}|\min_{((l,m),\cdot)\in\set E_k} \frac{|[\tilde{\mat P}_k]_{lm}-[\tilde{\mat P}_{k-1}]_{lm}|}{|[\tilde{\mat P}_{k-1}]_{lm}|}
    \end{align*} 
    if the relative deviation is used as basis of the trigger decision.
\end{theorem}
\begin{proof}
    See Appendix\,\ref{sec:most_changed_bounds}.
\end{proof}
Note that for the $N$-most-changed trigger, while the cardinality of $\set E_k'$, i.e., the number of elements that need to be transmitted, is always exactly $N$, no upper bound exist for the bound error $\|\mat W\odot(\hat{\mat P}_k-\mat P_k)\|_F$ unless additional assumptions are imposed on the sequence $\mat P_k$.
In contrast, the absolute- and relative-change triggers provide meaningful upper bounds on the bound error, but only the trivial one, i.e., the number of elements in the upper triangle of $\mat P_k$, on the number of transmitted index-value pairs.
Thus, the three triggers allow calibrating two complementary metrics: The bound error and the cardinality of the reduced event-set, i.e., the amount of transmitted data.
The different behaviors of the three proposed triggers are shown in Fig.~\ref{fig:different_triggers_illustration} for exemplary absolute deviations $|\mat P_k-\tilde{\mat P}_{k-1}|$ (relative deviations $|\mat P_k-\tilde{\mat P}_{k-1}|/|\tilde{\mat P}_{k-1}|$ for the relative-change trigger) from the transmitter buffer.

\subsection{Subset and Combined Triggers}
An interesting aspect of the above trigger-bounder framework that has not previously been discussed but follows from the special form of $\set F_k$ is the possibility to apply and even combine the basic (and other) triggers on a per-element basis:
For instance, the transmission of one subset of elements may be decided using an absolute-change trigger, while other subsets may be decided using an $N$-most-changed trigger.
In addition, some elements may always get sent, i.e., not use any trigger in the strict sense.
This is illustrated in Fig.~\ref{fig:mix_match_triggers_illustration}.
\begin{figure}
    \centering
    \definecolor{tab10_blue}{HTML} {1f77b4}
    \definecolor{tab10_orange}{HTML}{ff7f0e}
    \definecolor{tab10_green}{HTML}{2ca02c}
    \begin{tikzpicture}
        \node (mat) at (0,0) {$\begin{bmatrix} 
            [\mat P_k]_{11} & [\mat P_k]_{12} & [\mat P_k]_{13} & [\mat P_k]_{14} & [\mat P_k]_{15}\\[2pt]
            [\mat P_k]_{21} & [\mat P_k]_{22} & [\mat P_k]_{23} & [\mat P_k]_{24} & [\mat P_k]_{25}\\[2pt]
            [\mat P_k]_{31} & [\mat P_k]_{32} & [\mat P_k]_{33} & [\mat P_k]_{34} & [\mat P_k]_{35}\\[2pt]
            [\mat P_k]_{41} & [\mat P_k]_{42} & [\mat P_k]_{43} & [\mat P_k]_{44} & [\mat P_k]_{45}\\[2pt]
            [\mat P_k]_{51} & [\mat P_k]_{52} & [\mat P_k]_{53} & [\mat P_k]_{54} & [\mat P_k]_{55}\\ 
            \end{bmatrix}$};
        \draw[draw=tab10_blue,fill=tab10_blue,fill opacity=0.1] (-3.05,1.2) -- (3.05,1.2) -- (3.05,0.28) -- (-0.65,0.28) -- (-0.65,-1.19) -- (-3.05,-1.19) --cycle;
        \draw[draw=tab10_green,fill=tab10_green,fill opacity=0.1] (-0.475,0.215) -- (3.05,0.215) -- (3.05,-1.19) -- (-0.475, -1.19) --cycle;
        \draw[draw=tab10_orange,fill=tab10_orange,fill opacity=0.1] (0.65,1.225) -- (3.075,1.225) -- (3.075,0.255) -- (0.65, 0.255) --cycle;
        \draw[draw=tab10_orange,fill=tab10_orange,fill opacity=0.1] (-0.62,-1.22) -- (-3.08,-1.22) -- (-3.08,-0.24) -- (-0.62, -0.24) --cycle;
    \end{tikzpicture}
    \caption{Per-element application of triggers. Blue: absolute-change trigger. Orange and green: $N$-most-changed triggers for disjoint subsets of elements. Blue/orange overlap corresponds to a combined trigger.\label{fig:mix_match_triggers_illustration}}
\end{figure}
Triggers that apply only to a subset of elements in the covariance matrix are refered to as 'subset triggers' and the simultaneous application of multiple triggers to the same element is referred to as a 'combined trigger'.

\subsubsection*{Subset Triggers}
In the following, subset triggers, i.e., triggering mechanisms that are only applied to a subset of covariance matrix elements, are considered.
Such a subset of elements is defined by a subset of element indexes $\set G\subseteq\{1,\ldots,n\}^2$, where it is assumed that $(i,j)\in\set G\iff(j,i)\in\set G$.
In particular, the absolute- and relative-change triggers can be applied to a subset of elements simply by using their respective elementwise bounds $[\hat{\mat\Delta}_k]_{ij}$ only for the elements at indexes $(i,j)\in G$.
The elementwise bounds $[\hat{\mat\Delta}_k]_{ij}$ at indexes $(i,j)\notin\set G$ must then be defined in some other way, e.g., using the elementwise bounds of another subset trigger.
This follows immediately from the proofs of Theorems~\ref{thm:absolute_change_bounds} and \ref{thm:relative_change_bounds}, which are constructed in a strictly elementwise fashion.
More generally, any triggering mechanism where the triggering decision for the covariance matrix element at $(i,j)\in \set G$ depends only on the buffer and covariance matrix element at that same index, can be used as a subset trigger in this simple way.

As can easily be seen, this condition does not hold in the case of the $N$-most-changed trigger applied to a subset of elements, which would include the $N$ index-value pairs $((i,j),[\mat P_k]_{ij})$ in $\set E_k$ whose absolute (or relative) deviation from the respective buffer value ranks among the $N$ largest in the overlap of the upper triangle of $\mat P_k$ with $\set G$.
Nonetheless, the application of this triggering mechanism to a subset of covariance matrix elements is only slightly more involved than for the other mechanisms, as the following theorem shows.
\begin{theorem}[Subset-$N$-Most-Changed Bounds]\label{thm:subset_most_changed_bounds}
    Using the $N$-most-changed trigger for elements indexed by $\set G$, bounds on the the possible buffer errors at $(i,j)\in\set G$ are given by
    \begin{align*}
        [\hat{\mat\Delta}_k]_{ij} =
        \begin{cases}
            \hat\Delta_k & ,\,((i,j),\cdot)\notin\set E_k\\
            0 & ,\,\text{otherwise}
        \end{cases},
    \end{align*}
    where, if the absolute deviation is used for the trigger decision,
    \begin{align*}
        \hat\Delta_k &= \min_{\substack{((l,m),\cdot)\in\set E_k\\(l,m)\in\set G}} |[\tilde{\mat P}_k]_{lm}-[\tilde{\mat P}_{k-1}]_{lm}|
    \end{align*}
    and, if the relative deviation is used for the trigger decision,
    \begin{align*}
        \hat\Delta_k &= |[\tilde{\mat P}_k]_{ij}|\min_{\substack{((l,m),\cdot)\in\set E_k\\(l,m)\in\set G}} \frac{|[\tilde{\mat P}_k]_{lm}-[\tilde{\mat P}_{k-1}]_{lm}|}{|[\tilde{\mat P}_{k-1}]_{lm}|}.
    \end{align*}
\end{theorem}
\begin{proof}
    See Appendix\,\ref{sec:subset_most_changed_bounds}.
\end{proof}
Practically, this theorem demonstrates that the $N$-most-changed trigger (be it using absolute deviation or relative deviation), can be applied to a subset of covariance matrix elements, simply by replacing the minimum over all indexes in the computation of $\hat\Delta_k$ with a minimum over indexes in $\set G$.

\subsubsection*{Combined Triggers}
Combined triggers occur if several triggers are simultaneously applied to the same covariance matrix element.
This means that their triggering conditions must be simultaneously satisfied to trigger the transmission of the element they apply to.
Note that it is irrelevant if the involved triggers are subset triggers or not.
The next theorem gives a general way to construct the elementwise bounds $\hat{\mat\Delta}$ for a combined trigger from those of the constituent triggers.
Any triggering mechanisms that lead to the respective sets of possible buffer errors taking the form \eqref{eq:possible_errors}, not only those presented in this paper, can be combined in this way.
\begin{theorem}[Combined Bounds]\label{thm:combined_bounds}
    Let $[\hat{\mat\Delta}_{k,l}]_{ij}$ be the bounds on the possible buffer error at $(i,j)$ associated with the triggers $l=1,\ldots,m$ that are applied to the element at $(i,j)$. Then a bound on the possible buffer error at $(i,j)$ is given by 
    \begin{align*}
        [\hat{\mat\Delta}_k]_{ij} &= 
        \begin{cases}
            \max_{l=1,\ldots,m} [\hat{\mat\Delta}_{k,l}]_{ij} & ,\,((i,j),\cdot)\notin\set E_k\\
            0 & ,\,\text{otherwise}
        \end{cases}.
    \end{align*}
\end{theorem}
\begin{proof}
    See Appendix\,\ref{sec:combined_bounds}.
\end{proof}
While the above theorem is always applicable it may give more conservative elementwise bounds than necessary.
For instance, by combining an absolute-change and an (absolute-deviation-based) $N$-most-changed trigger\footnote{The following discussion applies analogously to the combination of a relative-change and a relative-deviation-based $N$-most-changed trigger.}, it is possible to create a combined trigger that behaves like an absolute-change trigger, but only transmits the up to $N$ most changed elements (from the upper triangle of $\mat P_k$) satisfying the absolute-change condition.
This effectively limits the number of index-value pairs that need to be transmitted to a specified maximum similarly to the $N$-most-changed trigger.
While elementwise bounds for this particular combination can be derived using Theorem~\ref{thm:combined_bounds}, the following theorem provides better bounds.
\begin{theorem}[Absolute-Change + $N$-Most-Changed Bounds]\label{thm:achange_nmost_bounds}
    Combining an absolute-change and an (absolute-deviation-based) $N$-most-changed trigger, bounds on the possible buffer errors are given by
    \begin{equation*}
        [\hat{\mat\Delta}_k]_{ij}=
        \begin{cases}
            \hat\Delta_k & ,\,((i,j),\cdot)\notin\set E_k\text{ and }|\set E_k'|=N\\
            [\mat T]_{ij} & ,\,((i,j),\cdot)\notin\set E_k\text{ and }|\set E_k'|<N\\
            0 & ,\,\text{otherwise}
        \end{cases},
    \end{equation*}
    where $\hat\Delta_k=\min_{((l,m),\cdot)\in\set E_k} |[\tilde{\mat P}_k]_{lm}-[\tilde{\mat P}_{k-1}]_{lm}|$.
\end{theorem}
\begin{proof}
    See Appendix\,\ref{sec:achange_nmost_bounds}.
\end{proof}
The above bounds on the possible buffer errors are better than those given by Theorem~\ref{thm:combined_bounds} because $\hat\Delta_k\geq[\mat T]_{ij}$ for any $((i,j),\cdot)\notin\set E_k$: If this was not the case an element would have been sent that has absolute deviation from its buffer value of less than $[\mat T]_{ij}$ -- a contradiction to the triggering conditions.

Note that the particular combination of an absolute-change and (absolute-deviation-based) $N$-most-changed trigger provides a somewhat limited upper bound on the bound error $\|\mat W\odot(\hat{\mat P}_k-\mat P_k)\|_F$, where $\mat W\in\{0,1\}^{n\times n}$ is used to select a submatrix of $\hat{\mat P}_k-\mat P_k$. 
This can be seen easily as follows: If less than $N$ index-value pairs were received $\hat{\mat\Delta}$ is essentially the same as for the absolute-change trigger. 
Therefore the same error bound $\|\mat W\odot(\hat{\mat P}_k-\mat P_k)\|_F\leq\|\mat W\odot(\diag(\mat T\vec 1)+\mat T)\|_F$ applies.
If $N$ index-value pairs were received $\hat{\mat\Delta}$ is identical to the $N$-most-changed trigger and no error bound can be given.

\subsection{Learning Thresholds from Application-Specific Data}\label{sec:optimal_threshs}
Learning trigger thresholds $\mat T$ from application-specific data aims to optimize the average of the transmitted data volume $|\set E_k'|$ and the average conservativeness of the bound $\hat{\mat P}_k$ at the receiver, in a particular application.
Here, similarly to \eqref{eq:bound_lp}, the conservativeness of a bound $\hat{\mat P}$ is measured in terms of its trace but is normalized to be in the range $[0,\infty)$.
The resulting measure of conservativeness $\trace(\hat{\mat P}-\mat P)/\trace(\mat P)$ is termed the 'relative conservativeness' of the upper bound $\hat{\mat P}$ with respect to the covariance matrix $\mat P$.
With these preliminaries, the dual aim pursued by learning thresholds is captured by the objective function ${L:\set S_n\rightarrow[0,\infty)}$ that is defined as
\begin{align}
    L(\mat T) = \sum_{(\mat P_{i,k})_{k=1}^{l_i}\in\set D}\frac{1}{l_i}\sum_{k=1}^{l_i}L_{i,k}^d(\mat T)+\lambda L_{i,k}^{rc}(\mat T)\label{eq:learning_objective}
\end{align}
and consists of 'data' objective terms $L_{i,k}^d$ and 'relative conservativeness' objective terms $L_{i,k}^{rc}$, with the tradeoff between the two determined by the parameter $\lambda\geq 0$.
Here, $\set D$ denotes a dataset of covariance matrix sequences from the intended application and $(\mat P_{i,k})_{k=1}^{l_i}$, $i=1,\ldots,|\set D|$ denote covariance matrix sequences of lengths $l_i$ from that dataset.
The 'data' terms $L_{i,k}^d$ count the number of elements transmitted in timestep $k$ of covariance matrix sequence $i$ and are given by
\begin{align*}
    L_{i,k}^d(\mat T) &= \sum_{\substack{l,m=1\\l\geq m}}^n [\mat\Gamma(\mat P_{i,k}-\tilde{\mat P}_{i,k},\mat T)]_{lm},
\end{align*}
where $[\mat\Gamma(\cdot)]_{lm}$ is one if the index-value pair $((l,m),[\mat P_{i,k}]_{lm})$ was transmitted and zero otherwise, and $\tilde{\mat P}_{i,k}$ is the buffer value associated with the $k$-th timestep of the $i$-th sequence.
The 'relative conservativeness' terms quantify the conservativeness of the bounder output $\hat{\mat P}$ (which depends on buffer matrices $\tilde{\mat P}_{i,k}$, $\tilde{\mat P}_{i,k-1}$, and thresholds $\mat T$) and are given by
\begin{align*}
    L_{i,k}^{rc}(\mat T) &= \sum_{\substack{l,m=1\\l\geq m}}^n \frac{\trace(\hat{\mat P}(\tilde{\mat P}_{i,k},\tilde{\mat P}_{i,k-1},\mat T)-\mat P_{i,k})}{\trace(\mat P_{i,k})}.
\end{align*}
Note that in all objective terms, $\hat{\mat P}_{i,k}=\hat{\mat P}(\tilde{\mat P}_{i,k},\tilde{\mat P}_{i,k-1},\mat T)$ and $\tilde{\mat P}_{i,k}$ are determined from the sequence $(\mat P_{i,k})_{k=1}^{l_i}$ using the proposed approach and the particular trigger considered.

Unfortunately, minimizing \eqref{eq:learning_objective} is not trivial: The function $\mat\Gamma(\cdot,\cdot,\cdot)$ is not continuous, which prevents optimization using gradient-based methods.
However, the threshold matrix $\mat T$ can be determined using a grid search for the $\mat T^*\in\set S_n$ that attains the lowest objective function value.
Since this approach quickly becomes computationally infeasible, even for low-dimensional covariance matrices, here, only the case where $\mat T=T\mat 1$, with $T$ a scalar threshold parameter, is considered.

\section{Results and Discussion}
In this section, the proposed event-triggered covariance matrix sequence transmission approach is evaluated in terms of its ability to reduce the amount of transmitted data and to produce bounds on the original covariance matrix that are not overly conservative.
In contrast to \cite{funk_conservative_2023}, instead of a single synthetic scenario, several scenarios arising from an extended Kalman filter (EKF) applied to real-world vehicle trajectories are considered for the evaluation of the approach.
The construction of these scenarios is described first, followed by an evaluation of the approach with the parameters of the used triggers selected manually.
Finally, results for the learning of trigger thresholds, as introduced in Sec.~\ref{sec:optimal_threshs}, are presented.

\subsection{Covariance Matrix Sequences from InD Dataset}
The dataset used in the following experiments, termed 'InD-EKF', is derived from the InD dataset \cite{bock_ind_2020} and consists of covariance matrix sequences as would occur in the centralized fusion of multiple EKF-based vehicle tracking estimates.
InD, which InD-EKF is based on, is a dataset of naturalistic vehicle trajectories (including positions, velocities, accelerations, orientations) recorded at four German intersections using a drone.
It comprises 11500 trajectories of cars, trucks, busses, pedestrians, and bicycles recorded at a rate of 25Hz with a typical position error of less than 10cm.
For the derivation of the 'InD-EKF' dataset only car trajectories with a length of at most 3000 timesteps (which excludes most stationary vehicles) are considered, resulting in a total of 7088 trajectories remaining.
The vehicle positions contained in each trajectory are additively disturbed with i.i.d. zero-mean Gaussian noise with variance $\sigma_p^2=0.1\,\textrm{m}^2$.
The disturbed positions are then processed by an EKF employing a vehicle model given by 
\begin{align}
    \begin{bmatrix}
        x_{k+1}      \\
        y_{k+1}      \\
        \theta_{k+1} \\
        v_{k+1}      \\
        \omega_{k+1}
    \end{bmatrix}
     & =
    \begin{bmatrix}
        x_k      \\
        y_k      \\
        \theta_k \\
        v_k      \\
        \beta\omega_k
    \end{bmatrix}+
    t_s\begin{bmatrix}
           v_k \cos\theta_k \\
           v_k \sin\theta_k \\
           \omega_k         \\
           0                \\
           0
       \end{bmatrix}+
    \begin{bmatrix}
        w_{x,k}      \\
        w_{y,k}      \\
        w_{\theta,k} \\
        w_{v,k}      \\
        w_{\omega,k}
    \end{bmatrix},\label{eq:vehicle_model_ekf}
\end{align}
where $x_k$, $y_k$ is the vehicle position, $\theta_k$ the heading angle,  $v_k$ the longitudinal velocity, $\omega_k$ the turn rate, $l=2.5\,\mathrm{m}$ the wheelbase, $\beta=0.9$ a parameter determining the decay of the turn rate, and $t_s=1/25\,\mathrm{s}$ the sampling period.
The terms $w_{x,k}$, $w_{y,k}$, $w_{\theta,k}$, $w_{v,k}$, and $w_{\omega,k}$ are zero-mean Gaussian noise with variances $\sigma_x^2=\sigma_y^2=10^{-3}\,\mathrm{m}^2$, $\sigma_\theta^2=10^{-3}\,\mathrm{rad}^2$, $\sigma_v^2=10^{-2}\,\frac{\mathrm{m}^2}{\mathrm{s}^2}$, and $\sigma_\omega^2=10^{-2}\,\frac{\mathrm{rad}^2}{\mathrm{s}^2}$.
The EKF uses a position measurement model which is given by
\begin{align}
    \vec z_k = \begin{bmatrix}
                   x_k \\
                   y_k
               \end{bmatrix}+
    \begin{bmatrix}
        v_{x,k} \\
        v_{y,k}
    \end{bmatrix},\label{eq:meas_model}
\end{align}
where $v_{x,k}$ and $v_{y,k}$ are zero-mean Gaussian noise terms with variance $\sigma_p^2=0.1\,\textrm{m}^2$.
The EKF estimate is initialized to $\hat{\vec x}_1=[[\vec z_1]_1,[\vec z_1]_2,0,0,0]^\top$, i.e., to the first available disturbed trajectory position, and its error covariance matrix to $\mat P_1=\diag(\sigma_p^2,\sigma_p^2,1,1,1)$.
The covariance matrix sequences obtained by the EKF constitute the InD-EKF dataset used for the following experiments.
A random 80-20 split of the entire dataset results in the InD-EKF training and test datasets.

\subsection{Manually Selected Trigger Parameters}\label{sec:manual_threshs_result}
\begin{figure*}
    \centering
    \includegraphics{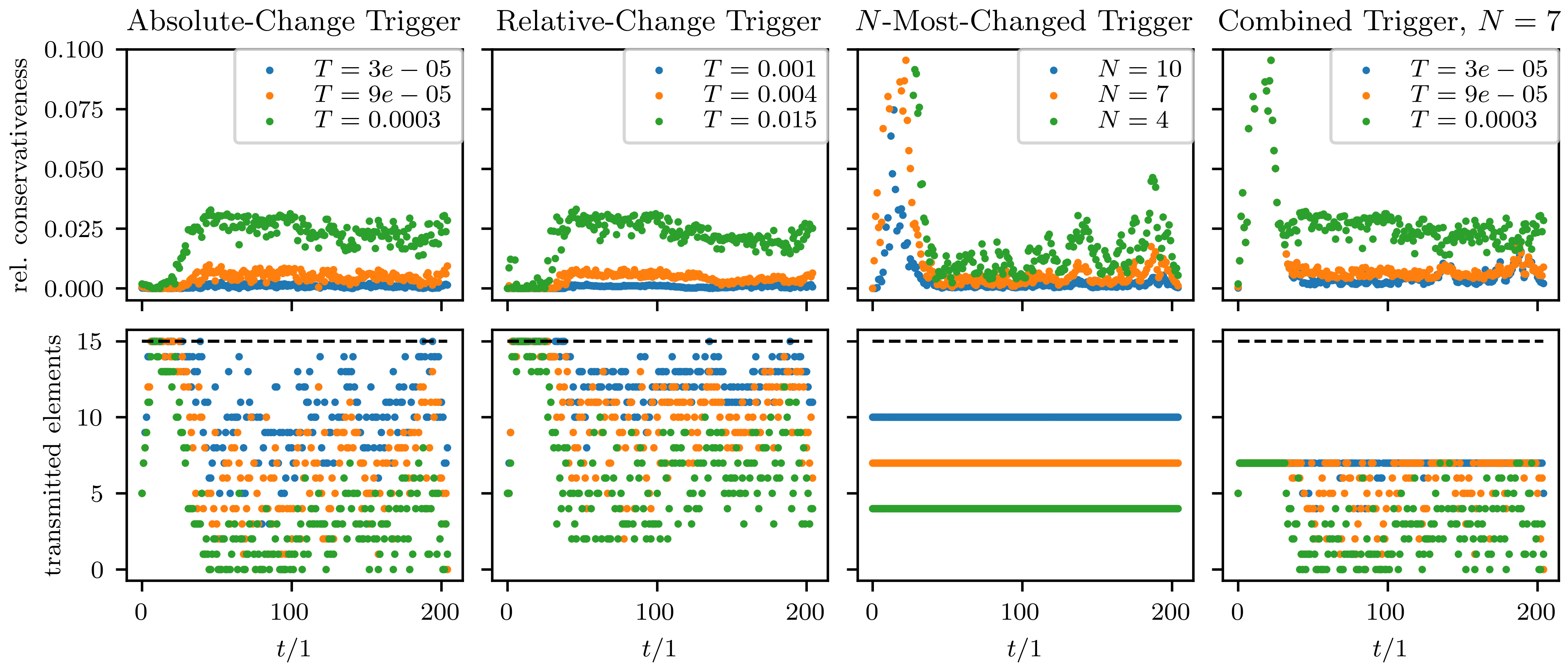}
    \caption{Relative conservativeness and number of transmitted elements of the absolute-change, relative-change, $N$-most-changed, and combined absolute-change and (absolute-deviation-based) $N$-most-changed trigger over time when applied to a single covariance matrix sequence from the InD-EKF dataset. The maximum number of elements in the upper triangle of the covariance matrices to be transmitted is shown as a dashed black line.\label{fig:single_sequence}}
\end{figure*}
\begin{figure*}
    \centering
    \includegraphics{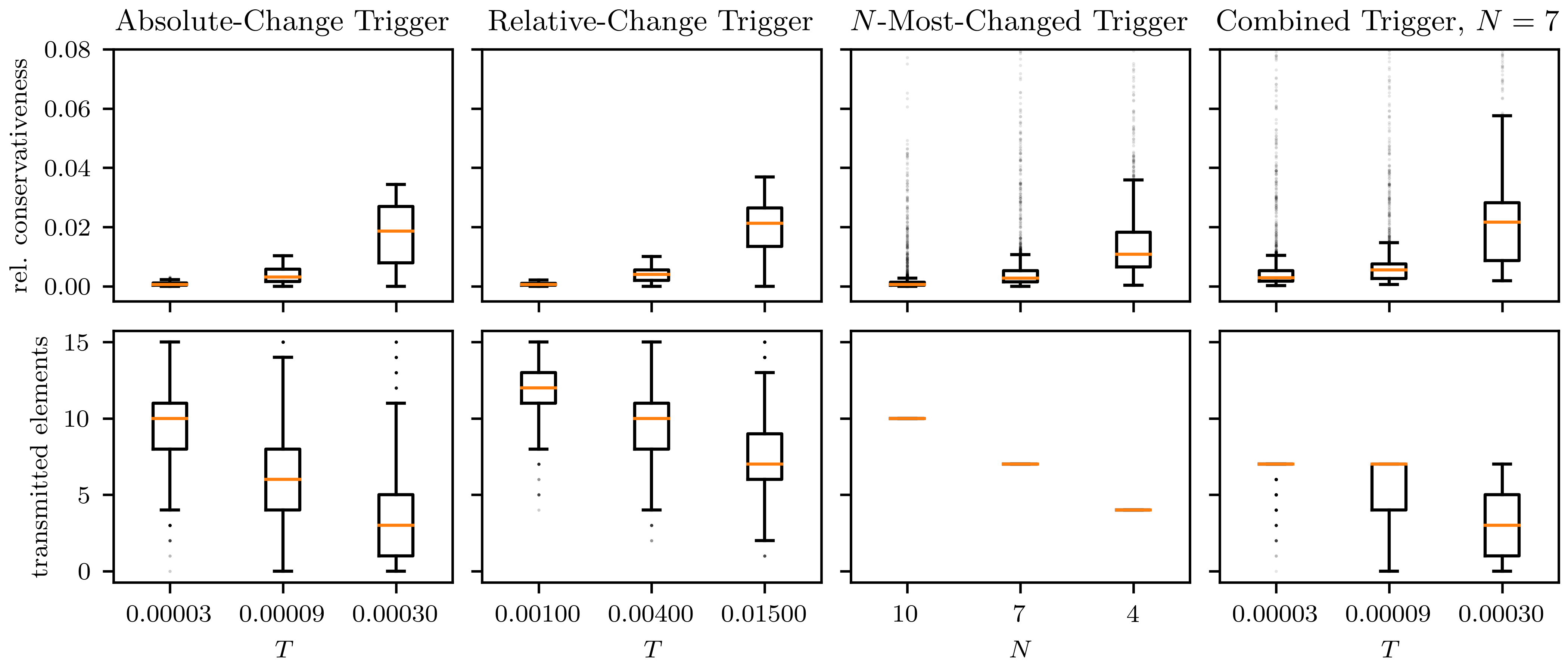}
    \caption{Summary statistics of relative conservativeness and number of transmitted elements of the absolute-change, relative-change, $N$-most-changed, and combined absolute-change and (absolute-deviation-based) $N$-most-changed trigger when applied to all test set covariance matrix sequences from the InD-EKF dataset. Outliers are shown as black dots with intensity proportional to their frequency of occurence.\label{fig:all_sequences}}
\end{figure*}
The initial series of experiments conducted focused on the qualitative characteristics of the proposed methodology when used with the different triggers and various parameter values. 
Specifically, a single covariance matrix sequence extracted from the InD-EKF test dataset was selected. 
Subsequently, the event-triggered transmission approach was applied to this sequence, and the relative conservativeness, as introduced in \ref{sec:optimal_threshs}, and the number of transmitted elements were recorded for each timestep.
For the absolute-change trigger, relative-change trigger, and the combined trigger (absolute-change trigger + absolute-deviation-based $N$-most-changed trigger), the threshold parameters were chosen uniformly, i.e., $\mathbf{T} = T\mathbf{1}$, where $T$ is a shared threshold value. 
The threshold values considered for absolute-change and combined triggers were $T=3 \cdot 10^{-5}$, $T=9 \cdot 10^{-5}$, and $T=3 \cdot 10^{-4}$, while those for the relative-change trigger were $T=0.001$, $T=0.004$, and $T=0.015$.
Regarding the $N$-most-changed trigger, the maximum numbers of transmitted elements considered were $N=4$, $N=7$, and $N=10$, while for the combined trigger only $N=7$ was considered.
In the subsequent set of experiments, the proposed methodology was systematically applied to all covariance matrix sequences within the InD-EKF test dataset. 
The relative conservativeness and the number of transmitted elements for each individual transmission were recorded, allowing for quantitative assessments of the approach's performance. 
The parameter values employed in these experiments mirrored those used in the qualitative analysis. 

\subsubsection*{Qualitative Results}
The recorded data pertaining to the relative conservativeness and the number of transmitted elements for the initial set of experiments are depicted in Fig.~\ref{fig:single_sequence}. 
It is evident from the figure that the relative conservativeness remains consistently low for both absolute- and relative-change triggers throughout most of the sequence, with a slight decrease in the initial part. 
Notably, relative conservativeness values approaching zero indicate that the bounds generated by the proposed approach have traces close to those of the original covariance matrices, implying only marginal over-conservativeness.
Due to positive semidefiniteness of the involved matrices, this also implies that the bound itself is close to the original covariance matrix\footnote{
Let $\trace(\hat{\mat P}-\mat P)=0$. This implies that the diagonal of $\hat{\mat P}-\mat P$ must be zero. The only PSD matrix with identically zero diagonal is the zero matrix.}.
The initially reduced relative conservativeness can be attributed to the rapid decrease in the initial EKF covariance matrix as new disturbed position measurements become available. 
This results in the transmission of all covariance matrix elements, as observed in the figure, leading to tight bounds and consequently low relative conservativeness. 
It is worth mentioning that the relative-change trigger, although producing bounds with a similar level of relative conservativeness as the absolute-change trigger, consistently transmits more covariance matrix elements. 
This is attributed to the relative-change trigger allowing arbitrarily small absolute covariance changes to be transmitted as long as the relative change over time is sufficient, in contrast to the absolute-change trigger, which has a fixed cutoff.
Both the $N$-most-changed and combined triggers exhibit similar initial behaviors in terms of relative conservativeness. A substantial peak is observed initially, owing to the limitation on the number of elements allowed for transmission. 
Following this peak, the relative conservativeness of the combined trigger remains mostly constant, whereas the $N$-most-changed trigger displays more variability in relative conservativeness due to the fixed number of elements allowed for transmission. 
This behavior contrasts with that of the absolute- and relative-change triggers, which primarily show variation in the number of transmitted elements rather than relative conservativeness.
As anticipated, lower threshold parameter values $T$ and larger values of $N$ result in increased data transmission and lower relative conservativeness.

\subsubsection*{Quantitative Results}
In Fig.~\ref{fig:all_sequences}, a statistical summary encompassing the median, first and third quartile, and $1.5\times$-interquartile range for both the relative conservativeness and the number of transmitted covariance elements, as observed in the second set of experiments is presented. 
These results give notable insights into the characteristics of the considered triggers at various parameter values. 
The absolute-change trigger demonstrates median data reduction ratios ranging from $33\%$ ($T=3\cdot10^{-5}$) to $80\%$ ($T=3\cdot10^{-4}$). 
Simultaneously, relative conservativeness values are close to zero, translating to a median increase in the bound trace between $0.1\%$ ($T=3\cdot10^{-5}$) and $1.9\%$ ($T=3\cdot10^{-4}$) compared to the original covariance matrix trace. 
At similar median relative conservativeness levels, the relative-change trigger achieves median data reduction ratios between $20\%$ ($T=10^{-3}$) and $53\%$ ($T=1.5\cdot10^{-2}$) for the considered threshold values, confirming the comparatively inferior performance in data reduction from the qualitative experiments. 
Therefore, the main advantage of the relative-change trigger is not its performance in terms of data reduction but the ease of choice of the threshold parameter independently from the scale of the expected covariance matrix sequences.
For the $N$-most-changed trigger, data reduction ratios are directly determined by the parameter $N$ and range from $33\%$ ($N=10$) to $73\%$ ($N=4$). 
The relative conservativeness results exhibit a median bound trace increase between $0.1\%$ ($N=10$) and $1.1\%$ ($N=4$) compared to the original covariance matrix trace, albeit with notable outliers. 
This phenomenon is attributed to the restricted number of transmitted elements during periods of rapid covariance matrix changes, as also evidenced in the qualitative results in Fig.~\ref{fig:single_sequence}.
The combined trigger yields relative conservativeness levels indicating a median bound trace increase between $0.3\%$ ($T=3\cdot10^{-5}$) and $2.2\%$ ($T=3\cdot10^{-4}$) compared to the original covariance matrix trace. 
Similar to the $N$-most-changed trigger and due to the same reasons, outliers are present, contributing to a larger spread in relative conservativeness results. 
The median data reduction ratio, consistently above $53\%$ for all threshold choices due to the specified $N=7$, ranges between $53\%$ ($T=3\cdot10^{-5}$) and $80\%$ ($T=3\cdot10^{-4}$).
The overarching observation from Fig.~\ref{fig:all_sequences} is that increased data transmission correlates with decreased relative conservativeness and vice versa. 
The above quantitative findings underscore that, on average and with the considered dataset, the proposed event-triggered covariance matrix transmission approach achieves practical data reduction while maintaining tight bounds.

\subsection{Learning Thresholds from Application-Specific Data}\label{sec:optimal_threshs_results}
\begin{figure}
    \centering
    \includegraphics{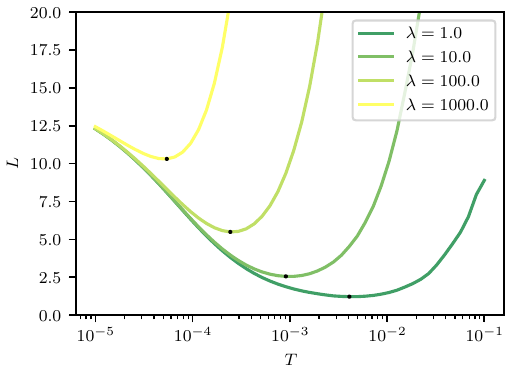}
    \caption{Training loss of absolute-change trigger for varying tradeoff parameter $\lambda\in[1,1000]$ and thresholds $\mat T=T\mat 1$, where $T$ is the shared threshold parameter. Black dots indicate the position of the minimum along the respective curve.\label{fig:uniform_train_loss}}
\end{figure}
\begin{figure*}
    \centering
    \includegraphics{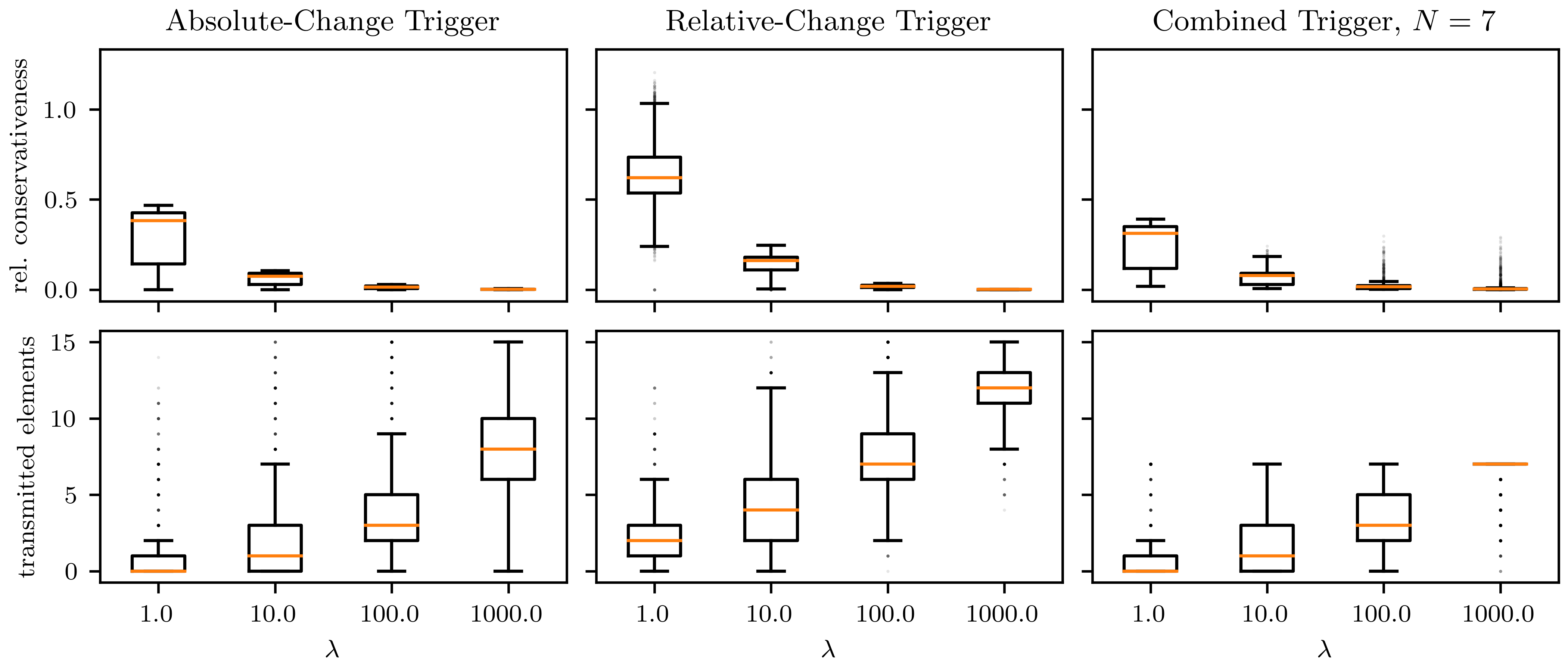}
    \caption{Summary statistics of relative conservativeness and number of transmitted elements of the absolute-change, relative-change, and combined absolute-change and (absolute-deviation-based) $N$-most-changed trigger when applied to all test set covariance matrix sequences from the InD-EKF test dataset with shared threshold $T$ optimized for different tradeoff parameters $\lambda$.\label{fig:uniform_thresh_nums_cons}}
\end{figure*}

In order to learn thresholds $\mat T=T\mat 1$ for the absolute-change, relative-change, and combined (absolute-change + absolute-deviation-based $N$-most changed trigger) triggers, where $T$ is a scalar value, in following set of experiments, the objective function \eqref{eq:learning_objective} was minimized on the InD-EKF training dataset using a grid search and varying values of the tradeoff parameter $\lambda$. 
Specifically, the tradeoff values $\lambda\in\{1,10,100,1000\}$ were considered, where $\lambda=0$ corresponds to an exclusive focus on minimizing the average number of transmitted elements and $\lambda=\infty$ corresponds to exclusive minimization of the average relative conservativeness.

Fig.~\ref{fig:uniform_train_loss} illustrates the threshold-dependent behavior of the objective function when the absolute-change trigger is employed on the InD-EKF training dataset with varying tradeoff parameters. 
As anticipated, larger values of the tradeoff parameter lead to smaller optimal thresholds. 
The objective function exhibits an unbounded increase as $T\rightarrow\infty$ due to the nature of the upper bound trace computation, where the thresholds of not transmitted elements enter linearly and larger thresholds eventually result in the transmission of no elements.
For $T\rightarrow 0$ the objective function approaches $15$, since for diminishing values of $T$, all $15$ elements in the upper triangle of the covariance matrices from the InD-EKF dataset (EKF state dimension is $5$) are transmitted, while the bound converges to the original covariance matrix.

The median, first and third quantile, as well as the $1.5\times$-interquartile range of the relative conservativeness and number of transmitted covariance matrix elements that result from using the optimal thresholds for different tradeoff-parameters on the InD-EKF test dataset are shown in Fig.~\ref{fig:uniform_thresh_nums_cons}.
Clearly, larger tradeoff parameters lead to more data being transmitted and lower relative conservativeness and vice versa.
The results also show that even when the thresholds are chosen optimally, the relative-change trigger not only transmits more covariance matrix elements but also exhibits slightly larger relative conservativeness than the absolute-change trigger.
The combined trigger exhibits behavior akin to the absolute-change trigger, with the differentiating factor being the number of transmitted elements being limited to $N=7$. 
This constraint contributes to an expanded range in relative conservativeness, particularly noticeable for larger values of~$\lambda$.

\section{Conclusion}
In sensor fusion systems, the combination of estimates and their corresponding covariance matrices necessitates the transmission of both, from the individual estimators producing them to a fusion center or other estimators performing decentralized fusion.
Notably, covariance matrices are the primary cause of data volume in such transmissions and consequently methodologies for the efficient transmission of these matrices are essential.
In safety-critical settings it is additionally of great importance that the involved covariance matrices are conservative, i.e., do not underestimate the uncertainty of their associated estimates, before and after transmission.
This work, which rigorously formalizes and generalizes \cite{funk_conservative_2023}, proposes an elementwise event-triggered approach for the transmission of covariance matrix sequences that is conservative and effective in terms of data reduction, while avoiding excessive conservativeness and computational expense.
In particular, in this work the method first introduced in \cite{funk_conservative_2023} is rederived in a more general setting that allows, e.g., for per-element trigger thresholds, per-element choice of trigger, and per-element combination of multiple triggers.
This provides great additional flexibility to adapt the method to the requirements (e.g., per-element accuracy specifications) of the application at hand.
In addition, this work introduces a new trigger whose trigger decisions are invariant to the scale of the covariance matrix sequences to be transmitted, a desirable property, and an approach to automatically learn optimal thresholds from a dataset of covariance matrix sequences for a given application.
Compared to \cite{funk_conservative_2023} experiments with covariance sequences based on real-world data are performed.
The experimental results underscore the quantitative performance of the proposed methodology, revealing an impressive $80\%$ median reduction in data volume, with a marginal $1.9\%$ median increase in the trace, here used as a measure of conservativeness, of the transmitted covariance matrix.
The experiments also answer questions of feasibility and efficacy of learning trigger thresholds, when a dataset of covariance matrix trajectories from an intended application is available, affirmatively.
Future work explores learning approaches for per-element thresholds, examining their potential advantages over uniform thresholds, especially when accuracies are not predetermined by specific requirements but are instead chosen to optimize the tradeoff between data reduction and conservativeness.
Tighter approximations of the cone of positive semidefinite than that of diagonal dominance used here are also an interesting research prospect, as they would directly improve the conservativeness of the method.
Finally, elementwise event-triggered covariance matrix transmission can be combined with event-based transmission of estimates.

\appendix
\subsection{Proof of Lemma~\ref{thm:dd_implies_psd}}\label{sec:dd_implies_psd}
\begin{theorem}[Gershgorin Discs]
    \label{thm:gerhsgorin}
    The union $\bigcup_{i=1}^n \set D_i$ of so-called \normalfont{Gershgorin discs} $\set D_i=\{\lambda\in\mathbb C||\lambda-[\mat A]_{ii}|\leq \sum_{j\neq i} |[\mat A]_{ij}|\}$ contains all eigenvalues of $\mat A\in\mathbb R^{n\times n}$ \cite{horn_matrix_2012}.
\end{theorem}

Using the above theorem the lemma can be proven as below.

\begin{proof}
    Since $\mat A$ is symmetric its eigenvalues are located on the real line.
    Furthermore, it follows from Theorem~\ref{thm:gerhsgorin} that they must be contained within the Gershgorin discs of $\mat A$.
    Each disc gives a lower bound $[\mat A]_{ii}-\sum_{j\neq i}|[\mat A]_{ij}|\leq \lambda$ on possible eigenvalues $\lambda\in\mathbb R$.
    Because of Definition~\ref{def:diagonal_dominance}, $[\mat A]_{ii}-\sum_{j\neq i}|[\mat A]_{ij}|\geq 0$ and consequently $\lambda\geq 0$ hold.
    Thus, all of $\mat A$'s eigenvalues are non-negative, i.e., $\mat A\succeq\mat 0$, as claimed.
\end{proof}

\subsection{Proof of Lemma~\ref{thm:feasible_sets_identical}}\label{sec:feasible_sets_identical}
\begin{proof}
    Let $\mat S_k\in\set A$ and choose $[\mat\Delta]_{ii}=-[\hat{\mat\Delta}_k]_{ii}$ for $i=1,\ldots,n$, $[\mat\Delta]_{ij}=\sgn([\mat S_k]_{ij})[\hat{\mat\Delta}_k]_{ij}$ for $i\neq j$. 
    Note that this choice of $\mat\Delta$ is in $\set F_k$. 
    From the above and the definition of $\set A$ $$-[\hat{\mat\Delta}_k]_{ii}+[\mat S_k]_{ii}\geq \sum_{\substack{j=1\\j\neq i}}^n[\hat{\mat\Delta}_k]_{ij}+\sum_{\substack{j=1\\j\neq i}}^n|[\mat S_k]_{ij}|$$ follows for $i=1,\ldots,n$.
    Hence, $\mat S_k\in\set B$ and $\set A\subseteq\set B$ hold.

    Let $\mat S_k\in\set B$ and choose an arbitrary $\mat\Delta\in\set F_k$. 
    From the definition of $\set F_k$ it follows that $[\hat{\mat\Delta}_k]_{ii}\geq-[\mat\Delta]_{ii}$ for $i=1,\ldots,n$.
    The triangle inequality and definition of $\set F_k$ yield $$|[\mat\Delta]_{ij}+[\mat S_k]_{ij}|\leq|[\mat\Delta]_{ij}|+|[\mat S_k]_{ij}|\leq[\hat{\mat\Delta}_k]_{ij}+|[\mat S_k]_{ij}|$$ for $j\neq i$.
    Lower bounding the defining inequalities of $\set B$ using the above inequalities gives $$[\mat S_k]_{ii}\geq-[\mat\Delta]_{ii}+\sum_{j\neq i}|[\mat\Delta]_{ij}+[\mat S_k]_{ij}|,$$
    for $i=1,\ldots,n$, which are exactly the defining inequalities of $\set A$.
    Since $\mat\Delta\in\set F_k$ was arbitrary, $\mat S_k\in\set A$ and $\set B\subseteq\set A$ hold.    
\end{proof}

\subsection{Proof of Theorem~\ref{thm:bound_lp_closed_form_solution}}\label{sec:bound_lp_closed_form_solution}
\begin{proof}
    LP \eqref{eq:bound_lp} may be viewed as a nested optimization problem, where one optimizes the diagonal elements $[\mat S_k]_{ii}$ of $\mat S_k$ first, holding the off-diagonal elements $[\mat S_k]_{ij}$, $i\neq j$, fixed, followed by optimizing the off-diagonal elements.
    The $[\mat S_k]_{ii}$ are bounded below by the inequality constraints.
    As they appear as $\sum_{i=1}^n[\mat S_k]_{ii}$ in the objective, the objective is minimized with respect to them if they attain their respective lower bounds.
    Therefore, the objective with the optimal values $[\mat S_k^*]_{ii}=\sum_{j=1}^n[\hat{\mat\Delta}_k]_{ij}+\sum_{\substack{j=1\\j\neq i}}^n|[\mat S_k]_{ij}|$ substituted for the $[\mat S_k]_{ii}$ is given by $\sum_{i=1}^n\left([\tilde{\mat P}_k]_{ii}+[\hat{\mat\Delta}_k]_{ii}+\sum_{\substack{j=1\\j\neq i}}^n|[\mat S_k]_{ij}|\right)$.
    The $[\mat S_k]_{ij}$, $i\neq j$, can now be seen to minimize the objective if they are zero.
    In summary, this implies that $[\mat S_k^*]_{ij}=0$, $i\neq j$, and $[\mat S_k^*]_{ii}=\sum_{j=1}^n[\hat{\mat\Delta}_k]_{ij}$ as claimed in the theorem.
\end{proof}

\subsection{Proof of Theorem~\ref{thm:bound_error_bound}}\label{sec:bound_error_bound}
\begin{proof}
    By the definitions of the bound and the buffer error $\|\mat W\odot(\hat{\mat P}_k-\mat P_k)\|_F=\|\mat W\odot(\mat\Delta_k+\mat S_k^*)\|_F$ holds.
    The latter is given by
    \begin{align*}
        &\|\mat W\odot(\mat\Delta_k+\mat S_k^*)\|_F^2 = \sum_{i,j=1}^n[\mat\Delta_k+\mat S_k^*]_{ij}^2[\mat W]_{ij}^2\\
        &= \sum_{i=1}^n[\mat\Delta_k+\mat S_k^*]_{ii}^2[\mat W]_{ii}^2+\sum_{\substack{i,j=1\\i\neq j}}^n[\mat\Delta_k+\mat S_k^*]_{ij}^2[\mat W]_{ij}^2
    \end{align*}
    Using the definitions of $\set F_k$ and $\mat S_k^*$ yields
    \begin{align*}
        \|\mat W\odot(\mat\Delta_k+\mat S_k^*)\|_F^2 &\leq \sum_{i=1}^n\left([\mat\Delta_k]_{ii}+\sum_{j=1}^n[\hat{\mat\Delta}_k]_{ij}\right)^2[\mat W]_{ii}^2\\
        &+\sum_{\substack{i,j=1\\i\neq j}}^n[\hat{\mat\Delta}_k]_{ij}^2[\mat W]_{ij}^2,
    \end{align*}
    where, due to $|[\mat\Delta_k]_{ii}|\leq[\hat{\mat\Delta}_k]_{ii}$ and $\hat{\mat\Delta}_k\geq\mat0$ the contents of the bracket are non-negative for $i=1,\ldots,n$.
    This implies $\left([\mat\Delta_k]_{ii}+\sum_{j=1}^n[\hat{\mat\Delta}_k]_{ij}\right)^2\leq\left([\hat{\mat\Delta}_k]_{ii}+\sum_{j=1}^n[\hat{\mat\Delta}_k]_{ij}\right)^2$.
    Written more compactly this results in the claimed bound.
\end{proof}

\subsection{Proof of Theorem~\ref{thm:absolute_change_bounds}}\label{sec:absolute_change_bounds}
\begin{proof}
    Each entry $[\mat P_k]_{ij}$ of $\mat P_k$ that was sent/received results in the corresponding $[\mat \Delta_k]_{ij}$ being zero, i.e., $|[\mat\Delta_k]_{ij}|$ is bounded by $[\hat{\mat\Delta}_k]_{ij}=0$.
    Each $[\mat P_k]_{ij}$ that was not sent satisfies
    \begin{equation*}
        \left|[\mat\Delta_k]_{ij}\right|=\left|[\mat P_k]_{ij}-[\tilde{\mat P}_k]_{ij}\right|=\left|[\mat P_k]_{ij}-[\tilde{\mat P}_{k-1}]_{ij}\right|\leq [\mat T]_{ij}
    \end{equation*}
    due to $[\tilde{\mat P}_k]_{ij}=[\tilde{\mat P}_{k-1}]_{ij}$ and the absolute-change trigger condition $|[\mat P_k]_{ij}-[\tilde{\mat P}_{k-1}]_{ij}|>[\mat T]_{ij}$.
\end{proof}

\subsection{Proof of Theorem~\ref{thm:relative_change_bounds}}\label{sec:relative_change_bounds}
\begin{proof}
    Each entry $[\mat P_k]_{ij}$ of $\mat P_k$ that was sent/received results in the corresponding $[\mat \Delta_k]_{ij}$ being zero, i.e., $|[\mat\Delta_k]_{ij}|$ is bounded by $[\hat{\mat\Delta}_k]_{ij}=0$.
    Each $[\mat P_k]_{ij}$ that was not sent satisfies
    \begin{align*}
        \left|[\mat\Delta_k]_{ij}\right|&=\left|[\mat P_k]_{ij}-[\tilde{\mat P}_k]_{ij}\right|=\left|[\mat P_k]_{ij}-[\tilde{\mat P}_{k-1}]_{ij}\right|\\
        &\leq [\mat T]_{ij}|[\tilde{\mat P}_{k-1}]_{ij}|=[\mat T]_{ij}|[\tilde{\mat P}_k]_{ij}|
    \end{align*}
    due to $[\tilde{\mat P}_k]_{ij}=[\tilde{\mat P}_{k-1}]_{ij}$ and the relative-change trigger condition $|[\mat P_k]_{ij}-[\tilde{\mat P}_{k-1}]_{ij}|>[\mat T]_{ij}|[\tilde{\mat P}_{k-1}]_{ij}|$.
\end{proof}

\subsection{Proof of Theorem~\ref{thm:most_changed_bounds}}\label{sec:most_changed_bounds}
\begin{proof}
    Each entry $[\mat P_k]_{ij}$ of $\mat P_k$ that was sent/received results in the corresponding $[\mat\Delta_k]_{ij}$ being zero, i.e., $|[\mat\Delta_k]_{ij}|$ is bounded by $[\hat{\mat\Delta}_k]_{ij}=0$. 
    The $[\mat P_k]_{ij}$ of the upper triangle that were not received can not have been among the $N$ elements of the upper triangle with the largest absolute buffer deviation $|[\mat P_k]_{ij}-[\tilde{\mat P}_{k-1}]_{ij}|$ (or relative buffer deviation $|[\mat P_k]_{ij}-[\tilde{\mat P}_{k-1}]_{ij}|/|[\tilde{\mat P}_{k-1}]_{ij}|$).
    Therefore, their absolute (or relative) buffer deviation must have been less than that of the upper triangle elements that were received, i.e.,
    \begin{equation*}
        |[\mat P_k]_{ij}-[\tilde{\mat P}_{k-1}]_{ij}|\leq \min_{((l,m),\cdot)\in\set E_k'} |[\mat P_k]_{lm}-[\tilde{\mat P}_{k-1}]_{lm}|
    \end{equation*}
    if the absolute deviation was used or
    \begin{equation*}
        \frac{|[\mat P_k]_{ij}-[\tilde{\mat P}_{k-1}]_{ij}|}{|[\tilde{\mat P}_{k-1}]_{ij}|}\leq \min_{((l,m),\cdot)\in\set E_k'} \frac{|[\mat P_k]_{lm}-[\tilde{\mat P}_{k-1}]_{lm}|}{|[\tilde{\mat P}_{k-1}]_{lm}|}
    \end{equation*}
    if the relative deviation was used, for each $((i,j),\cdot)\notin\set E_k'$.
    Note that here $\set E_k'$ can be replaced by $\set E_k$ due to symmetry of $\mat P_k$ and $\tilde{\mat P}_{k-1}$.
    In addition, $[\tilde{\mat P}_k]_{lm}=[\mat P_k]_{lm}$ for all $((l,m),\cdot)\in\set E_k$.
    This implies that if the absolute deviation was used
    \begin{equation*}
        |[\mat P_k]_{lm}-[\tilde{\mat P}_{k-1}]_{lm}|=|[\tilde{\mat P}_k]_{lm}-[\tilde{\mat P}_{k-1}]_{lm}|
    \end{equation*}
    and if the relative deviation was used
    \begin{equation*}
        \frac{|[\mat P_k]_{lm}-[\tilde{\mat P}_{k-1}]_{lm}|}{|[\tilde{\mat P}_{k-1}]_{lm}|}=\frac{|[\tilde{\mat P}_k]_{lm}-[\tilde{\mat P}_{k-1}]_{lm}|}{|[\tilde{\mat P}_{k-1}]_{lm}|}
    \end{equation*}
    hold for $((l,m),\cdot)\in\set E_k$.
    Thus, the absolute buffer error for a not received element $((i,j),[\mat P_k]_{ij})\notin\set E_k$ is 
    \begin{align*}
        |[\mat\Delta_k]_{ij}|&=|[\mat P_k]_{ij}-[\tilde{\mat P}_k]_{ij}| = |[\mat P_k]_{ij}-[\tilde{\mat P}_{k-1}]_{ij}| \\
        &\leq \min_{((l,m),\cdot)\in\set E_k} |[\tilde{\mat P}_k]_{lm}-[\tilde{\mat P}_{k-1}]_{lm}|,
    \end{align*}
    and the relative buffer error for such an element is
    \begin{align*}
        \frac{|[\mat\Delta_k]_{ij}|}{|[\tilde{\mat P}_k]_{ij}|} &= \frac{|[\mat P_k]_{ij}-[\tilde{\mat P}_k]_{ij}|}{|[\tilde{\mat P}_k]_{ij}|} = \frac{|[\mat P_k]_{ij}-[\tilde{\mat P}_{k-1}]_{ij}|}{|[\tilde{\mat P}_{k-1}]_{ij}|} \\
        &\leq \min_{((l,m),\cdot)\in\set E_k} \frac{|[\tilde{\mat P}_k]_{lm}-[\tilde{\mat P}_{k-1}]_{lm}|}{|[\tilde{\mat P}_{k-1}]_{lm}|},
    \end{align*}
    where the above observations and the fact that $[\tilde{\mat P}_k]_{ij}=[\tilde{\mat P}_{k-1}]_{ij}$ for a not received element $[\mat P_k]_{ij}$ were used.
\end{proof}

\subsection{Proof of Theorem~\ref{thm:subset_most_changed_bounds}}\label{sec:subset_most_changed_bounds}
\begin{proof}
    The proof of Theorem~\ref{thm:most_changed_bounds} can be reused almost entirely. 
    It is only necessary to realize that the not received elements $[\mat P_k]_{ij}$ cannot have been among the $N$ elements of the upper triangle with the largest absolute (or relative) buffer deviation \textit{within the considered subset of elements} defined by $\set G$.
    Therefore, their absolute (or relative) buffer deviation must have been less than that of the upper triangle elements in said subset that were received, i.e., 
    \begin{equation*}
        |[\mat P_k]_{ij}-[\tilde{\mat P}_{k-1}]_{ij}|\leq \min_{\substack{((l,m),\cdot)\in\set E_k'\\(l,m)\in\set G}} |[\mat P_k]_{lm}-[\tilde{\mat P}_{k-1}]_{lm}|
    \end{equation*}
    if the absolute deviation was used or
    \begin{equation*}
        \frac{|[\mat P_k]_{ij}-[\tilde{\mat P}_{k-1}]_{ij}|}{|[\tilde{\mat P}_{k-1}]_{ij}|}\leq \min_{\substack{((l,m),\cdot)\in\set E_k'\\(l,m)\in\set G}} \frac{|[\mat P_k]_{lm}-[\tilde{\mat P}_{k-1}]_{lm}|}{|[\tilde{\mat P}_{k-1}]_{lm}|}
    \end{equation*}
    if the relative deviation was used.
    The remainder of the proof follows exactly the same pattern as that of Theorem~\ref{thm:most_changed_bounds}.
\end{proof}

\subsection{Proof of Theorem~\ref{thm:combined_bounds}}\label{sec:combined_bounds}
\begin{proof}
    If $((i,j),\cdot)\in\set E_k$ then $[\mat P_k]_{ij}$ was received so the buffer error is $[\mat\Delta_k]_{ij}=0$ and $[\hat{\mat\Delta}_k]_{ij}=0$ is an upper bound for its absolute value.
    If $((i,j),\cdot)\notin\set E_k$ then $[\mat P_k]_{ij}$ must have not satisfied the triggering condition of at least one of the triggers.
    This implies that the possible buffer errors are bounded by at least one of the bounds associated with the triggers whose conditions were not satisfied.
    The maximum of all trigger bounds includes the bounds of the triggers whose conditions were not satisfied and is thus guaranteed to bound the possible buffer errors.
\end{proof}

\subsection{Proof of Theorem~\ref{thm:achange_nmost_bounds}}\label{sec:achange_nmost_bounds}
\begin{proof}
    Assume that $|\set E_k'|<N$, then all elements of $\mat P_k$ that satisfy the absolute-change criterion are included in $\set E_k$. 
    Therefore, the same bounds as for the absolute-change trigger apply.
    Now assume that $|\set E_k'|=N$, then some elements that satisfy the absolute-change criterion may be excluded from $\set E_k$ and moreover it is unknown which elements that may be.
    Therefore, the best that can be done is to bound the possible buffer errors of the elements not in $\set E_k$ using the same approach as in the most changed trigger.
\end{proof}

\bibliographystyle{IEEEtran}
\bibliography{refs_new}

\begin{IEEEbiography}[{\includegraphics[width=1in,height=1.25in,clip,keepaspectratio]{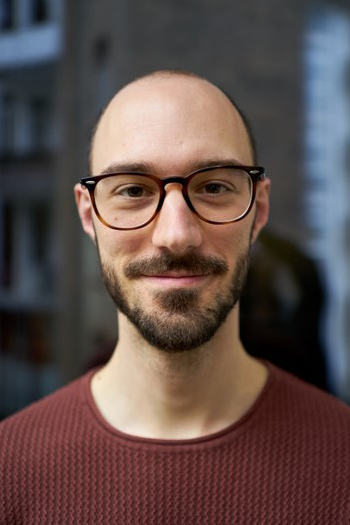}}]{Christopher Funk} (Student Member, IEEE) received the M.Sc. degree in electrical engineering and information technology from the Karlsruhe Institute of Technology (KIT), Karlsruhe, Germany, in 2019. He is currently working towards the Ph.D. degree in computer science with the Autonomous Multisensor Systems (AMS) group, Otto von Guericke University,  Magdeburg, Germany. His research interests include multisensor data fusion, data compression for decentralized multisensor data fusion, and optimization-based methods in multi-object tracking.
\end{IEEEbiography}
\begin{IEEEbiography}[{\includegraphics[width=1in,height=1.25in,clip,keepaspectratio]{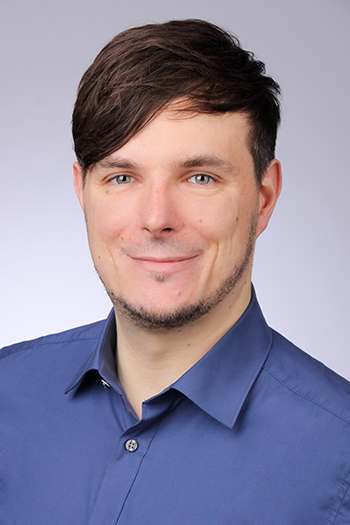}}]{Benjamin Noack} (Member, IEEE) received the diploma in computer science from the Karlsruhe Institute of Technology (KIT), Karlsruhe, Germany, in 2009, and the Ph.D. degree in computer science from the Intelligent Sensor-Actuator-Systems Laboratory, KIT, in 2013. He is currently a Professor of Computer Science with the Otto von Guericke University Magdeburg, Magdeburg, Germany, and the Head of the Autonomous Multisensor Systems (AMS) Group. His research interests include multisensor data fusion, distributed and decentralized Kalman filtering, combined stochastic and set-membership approaches to state estimation, and event-based systems.
\end{IEEEbiography}

\end{document}

%% file: approach_structure.tex
\begin{tikzpicture}
    \node (input) at (3.1,-0.9) {$\mat P_k$};
    \node[draw,rounded corners] (buffer) at (1.1,0) {buffer};
    \node[draw,rounded corners] (trigger) at (3.1,0) {trigger};
    \node[draw,rounded corners] (rbuffer) at (5.2,0) {buffer};
    \node[draw,rounded corners] (bounder) at (7.0,0) {bounder};
    \node (output) at (7.0,-0.9) {$\hat{\mat P}_k$};
    \coordinate (mid) at (3.8,0);
    \coordinate (abovemid) at (3.8,0.6);
    \coordinate (abovebuffer) at (1.1,0.6);
    \draw[->] (input) -- (trigger);
    \draw[->] (bounder) -- (output);
    \draw[->] (buffer) -- node[above] (bar) {$\tilde{\mat P}_{k-1}$} (trigger);
    \draw[->] (mid) -- (abovemid) -- (abovebuffer) -- (buffer.north);
    \draw[->] (trigger) -- (mid) -- node[above]{$\set E_k'$} (rbuffer);
    \draw[->] (rbuffer) -- node[above] (foo) {$\tilde{\mat P}_k$} (bounder);
    \node[fit=(buffer)(trigger)(mid)(bar),draw,rounded corners,label=transmitter] {};
    \node[fit=(rbuffer)(bounder)(foo),draw,rounded corners,label=receiver] {};
\end{tikzpicture}

%% file: example_transmission.tex
\begin{tikzpicture}
\node at (0,1) {$k=0$};
\node at (2,1) {$k=1$};
\node at (4,1) {$k=2$};

\node[align=center] at (-2.25,0) {$\mathbf{P}_k$\\(sender)};
\node at (2,0) {$\begin{bmatrix}\textcolor{red}{1.5} & \textcolor{red}{0.5}\\\textcolor{red}{0.5} & 1.2\end{bmatrix}$};
\node at (4,0) {$\begin{bmatrix}1.5 & 0.5\\0.5 & \textcolor{red}{0.5}\end{bmatrix}$};

\node[align=center] at (-2.25,-1.5) {$\tilde{\mathbf{P}}_k$\\(buffer)};
\node at (0,-1.5) {$\begin{bmatrix}1.0 & 0.0\\0.0 & 1.0\end{bmatrix}$};
\node at (2,-1.5) {$\begin{bmatrix}\textcolor{red}{1.5} & \textcolor{red}{0.5}\\\textcolor{red}{0.5} & 1.0\end{bmatrix}$};
\node at (4,-1.5) {$\begin{bmatrix}1.5 & 0.5\\0.5 & \textcolor{red}{0.5}\end{bmatrix}$};

\node[align=center] at (-2.25,-3) {$\hat{\mathbf{P}}_k$\\(receiver)};
\node at (2,-3) {$\begin{bmatrix}1.5 & 0.5\\0.5 & \textcolor{red}{1.7}\end{bmatrix}$};
\node at (4,-3) {$\begin{bmatrix}\textcolor{red}{2.0} & 0.5\\0.5 & \textcolor{red}{1.0}\end{bmatrix}$};

\draw[->,thick] (1.4,-0.5) -- node[right]{transmit} (1.4,-1);
\draw[->,thick] (3.4,-0.5) -- node[right]{transmit} (3.4,-1);
\draw[->,thick] (1.5,-2) -- node[right]{bound} (1.5,-2.5);
\draw[->,thick] (3.5,-2) -- node[right]{bound} (3.5,-2.5);
\end{tikzpicture}